\documentclass{article}
\usepackage{ijcai25}

\usepackage{times}
\usepackage{soul}
\usepackage{url}
\usepackage[hidelinks]{hyperref}
\usepackage[utf8]{inputenc}
\usepackage[T1]{fontenc}
\usepackage[small]{caption}
\usepackage{graphicx}
\usepackage{amsmath}
\usepackage{amsthm}
\usepackage{amsfonts}
\usepackage{amssymb}
\usepackage{booktabs}
\usepackage[ruled,vlined,linesnumbered]{algorithm2e}
\usepackage[inline,shortlabels]{enumitem}
\usepackage[labelformat=simple]{subcaption}

\usepackage{scalerel}
\usepackage{multicol}
\usepackage{tikz}
\usepackage{tikzsymbols}
\usetikzlibrary{patterns}
\usepackage[switch]{lineno}
\usepackage[capitalize, noabbrev]{cleveref}
\DeclareMathOperator\support{support}

\usepackage[textsize=small]{todonotes}

\newcommand*{\state}{\mathbf{x}}
\newcommand*{\statespace}{\mathbb{X}}
\newcommand*{\action}{\mathbf{u}}
\newcommand*{\actionspace}{\mathbb{U}}
\newcommand*{\noise}{\mathbf{w}}
\newcommand*{\noisespace}{\mathbb{W}}

\newcommand{\ProbabilityMeasure}[1][\Sigma]{\mathbb{P}_{\state_0}^{#1}}
\newcommand{\Expectation}[1][\Sigma]{\mathbb{E}_{\state_0}^{#1}}

\newcommand*{\initial}{\statespace_{0}}
\newcommand*{\reach}{\statespace_{\star}}
\newcommand*{\avoid}{\statespace_{\oslash}}

\newcommand*{\globalmark}{\star}

\newcommand*{\System}{\Sigma}
\newcommand*{\ChangedSystem}{\widetilde{\System}}

\newcommand*{\ChangedPath}{\widetilde{\pi}^{\globalmark}}

\newcommand*{\Specification}{\varphi}

\newcommand*{\Dynamics}{f}
\newcommand*{\ChangedDynamics}{\widetilde{\Dynamics}}

\newcommand*{\Certificate}[1][]{C_{#1}}
\newcommand*{\Policy}[1][]{\pi_{#1}}
\newcommand*{\Bound}[1][]{\rho_{#1}}

\newcommand*{\GlobalPolicy}{\Policy^{\globalmark}}

\newcommand*{\ChangedSet}{\statespace_{?}}
\newcommand*{\InfimumChangedSet}{\inf\{\Certificate(\state)\mid \state\in\ChangedSet\}}

\newcommand*{\ChangedPolicy}[1][]{\widetilde\Policy_{#1}}
\newcommand*{\ChangedBound}[1][]{\widetilde\Bound_{#1}}

\newcommand*{\ChangedGlobalPolicy}{\ChangedPolicy^{\globalmark}}

\newcommand*{\InitialVertex}{v_{0}}
\newcommand*{\TargetVertex}{v_{\star}}

\renewcommand*{\InfimumChangedSet}{\inf\{\Certificate(\ChangedSet)\}}

\usepackage{silence}
\newtheorem{example}{Example}
\newtheorem{theorem}{Theorem}
\newtheorem{definition}{Definition}
\newtheorem{problem}{Problem}
\newtheorem{lemma}{Lemma}

\title{VeRecycle: Reclaiming Guarantees from Probabilistic Certificates \\for Stochastic Dynamical Systems after Change}

\author{
Sterre Lutz
\and
Matthijs T.J. Spaan
\And
Anna Lukina \\
\affiliations
Delft University of Technology, The Netherlands\\
\emails
\{s.lutz, m.t.j.spaan, a.lukina\}@tudelft.nl
}

\begin{document}

\maketitle

\begin{abstract}
    Autonomous systems operating in the real world encounter a range of uncertainties. Probabilistic neural Lyapunov certification is a powerful approach to proving safety of nonlinear stochastic dynamical systems. When faced with changes beyond the modeled uncertainties, e.g., unidentified obstacles, probabilistic certificates must be transferred to the new system dynamics. However, even when the changes are localized in a known part of the state space, state-of-the-art requires complete re-certification, which is particularly costly for neural certificates. We introduce VeRecycle, the first framework to formally reclaim guarantees for discrete-time stochastic dynamical systems. VeRecycle efficiently reuses probabilistic certificates when the system dynamics deviate only in a given subset of states. We present a general theoretical justification and algorithmic implementation. Our experimental evaluation shows scenarios where VeRecycle both saves significant computational effort and achieves competitive probabilistic guarantees in compositional neural control.
\end{abstract}

\section{Introduction}
\label{sec:introduction}

Autonomous systems often exhibit nonlinear stochastic behavior. 
In safety-critical applications (e.g., robot navigation), it is crucial to provide provable guarantees for autonomous systems, both in terms of performance (e.g., reaching a target) and safety (e.g., avoiding material damage)~\cite{kwiatkowska_when_2023}. The latter is particularly challenging to guarantee for the multitude of uncertainties of the real world. For instance, a spill of water may violate the friction threshold for a robot to move at the required safe speed.

Designing safe control for a stochastic nonlinear dynamical system, as a general model for many real-world autonomous systems, is enabled by the Lyapunov theory~\cite{kalman1960}. Initially developed for proving the stability of deterministic dynamical systems, Lyapunov theory prescribes the construction of a potential-energy function, which characterizes system dynamics and, when obeying a set of given conditions, becomes a formal certificate for the desired specification. Reaching the target while avoiding unsafe states can be naturally formulated as Lyapunov stability: the system's potential energy decreases towards the target region and increases around unsafe regions. Lyapunov certification extends to stochastic dynamical systems~\cite{blumenthal1968} and to proving reach-avoid properties thereof~\cite{prajna2004stochastic}.  

Overwhelming success of neural networks in visual recognition and especially feasibility of their verification~\cite{katz2017,gehr2018,zhang2018efficient,kouvaros2021,wu2024marabou} galvanized the research community into adoption of neural components in control design. Verification of dynamical systems under neural control called for an alternative to sum-of-square programming~\cite{papachristodoulou2002,topcu2008}, limited to polynomial dynamics. Neural Lyapunov certification proved successful for deterministic systems~\cite{richards2018,chang2019,abate2020formal,edwards2023,dawson2023safe}. Recently, inspired by termination analysis of stochastic programs~\cite{giacobbe2022termination}, probabilistic neural Lyapunov certification solidified its place as the state-of-the-art for reach-avoid properties of stochastic dynamical systems under neural control~\cite{zikelic_learning_2023,ansaripour2023,chatterjee2023,mathiesen2023,abate2024stochastic,badings_learning-based_2024}.   

For a given model of dynamics and a control policy, state-of-the-art probabilistic reach-avoid certificates guarantee that with probability above a given threshold, the system will reach a target set of states while avoiding unsafe states. 
However, in scenarios where the system dynamics deviate from their stochastic model, even if only in a known subset of system states, the originally computed certificate requires re-validation to transfer its guarantee. 
Yet, there exists no principled way to infer if the changes in the dynamics affect the guarantees or to reuse the corresponding certificate, without repeating costly (neural) certification.

\begin{example}\label{ex:VeRecycle-examples}
A robot navigates a warehouse with slippery floors, meaning the result of the robot's actions is probabilistic. 
To safely reach room 8 starting from room 0 (dotted squares in Figure~\ref{fig:VeRecycle-examples-certificate} mark the centers of the rooms), the robot is controlled by a composition of neural policies, including one guiding it from room 4 to 5, avoiding collision with the walls.
\Cref{fig:VeRecycle-examples-certificate} visualizes a probabilistic neural certificate, which guarantees that starting anywhere close to the center of room 4, the robot will safely reach the center of room 5 with probability at least $0.88$. 
Suppose that an unidentified obstacle appeared in room 2 and the whole room is now deemed unsafe. Although the obstacle is not interfering with the given subtask, visual inspection shows the certificate is now invalid (dashed line crosses the area around the obstacle). 

\begin{figure}[t]
\centering
        \includegraphics[width=\columnwidth]{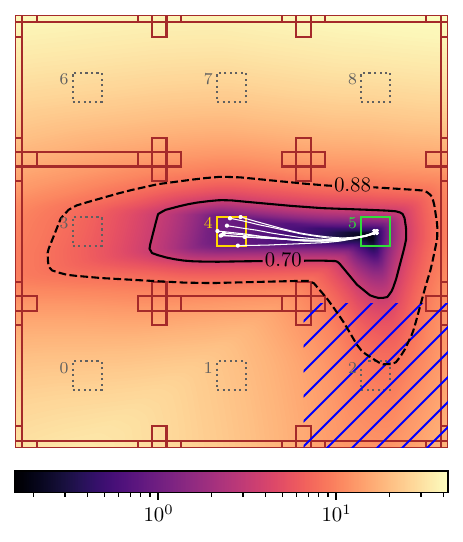}
        \caption{A probabilistic certificate (color gradient) for a navigation subtask \colorbox{magenta!50!blue}{${\color{yellow}4\:\square}$} (yellow) $\rightarrow$ \colorbox{magenta!20!blue}{${\color{green}5\:\square}$} (green), with darker color indicating higher probability of the system (with trajectories sampled in white) to complete the subtask while avoiding the walls (outlined in red). Dashed line shows original certified probability threshold, solid line shows threshold reclaimed by VeRecycle after a localized change in the system's dynamics (slanted blue pattern).}\label{fig:VeRecycle-examples-certificate}
\end{figure}
\end{example}

In this work, we propose the first, to the best of our knowledge, theoretical justification for reclaiming probabilistic reach-avoid guarantees for neural control policies of discrete-time stochastic dynamical systems affected by a localized change, i.e., for which the dynamics changed in a given subset of the state space. In that, we assume no knowledge about the nature of the change.

Our approach, called VeRecycle, takes as input a neural control policy and a probabilistic neural reach-avoid certificate with a particular probability threshold. For an arbitrary, possibly unknown, change in the system dynamics affecting a known subset of the state space, VeRecycle automatically infers a new certified reach-avoid probability threshold.

In Example~\ref{ex:VeRecycle-examples}, VeRecycle quickly identifies the affected certificate and updates its threshold to $0.70$ (solid line has no intersection with the unsafe area), without any additional neural training. This enables further use of existing neural policies and reclaiming the global probabilistic guarantee.

VeRecycle magnifies the advantages of tackling the problems with innate modular structure through the prism of localized analysis, for which compositional approaches to control and---with VeRecycle---certification demonstrate greater scalability~\cite{neary_verifiable_2023,zikelic_compositional_2023,zhang2023compositional,delgrange2024controller}. In compositional control, our approach 1) automatically identifies which components of the compositional policy are affected by the change and 2) eliminates---when possible---the need for costly repairs of neural control policies, certificates, or system models by using unaffected components in a new compositional policy.

To summarize, we propose VeRecycle, a framework to formally reclaim guarantees from probabilistic neural certificates for discrete-time stochastic dynamical systems after localized changes. 
First, we formally prove that sound guarantees are inferred by VeRecycle from (approximated) infima of the original certificates. Second, our experimental evaluation demonstrates scenarios where VeRecycle's guarantees are competitive---particularly for compositional control---compared to complete re-certification, with computational cost reduced by three orders of magnitude. Thus, VeRecycle enables higher flexibility and scalability of the state-of-the-art probabilistic reach-avoid certification of discrete-time stochastic dynamical systems.

\section{Preliminaries}
\label{sec:preliminaries}
We consider discrete-time stochastic dynamical systems $\Sigma = (\statespace, \actionspace, \noisespace, f,\mu)$, defined by the following equation $\forall\:t\in\mathbb{N}_{>0}$: 
\begin{equation}\label{eq:system}
    \state_{t+1} = f(\state_{t}, \action_{t}, \noise_{t}), \quad \noise_{t} \sim \mu,
\end{equation}
where $\state_{t} \in \statespace \subseteq \mathbb{R}^n$ is the $n$-dimensional state of $\Sigma$ at time $t$ with $\state_0$ as the initial state,
$\action_{t} \in \actionspace \subseteq \mathbb{R}^m$ is the $m$-dimensional control input $\Sigma$ at time $t$, 
and $\noise_{t} \in \noisespace \subseteq \mathbb{R}^p$ is the $p$-dimensional realization of stochastic disturbance of $\Sigma$ at time $t$. 
Transitions between states follow the dynamics function
$f: \statespace \times \actionspace \times \noisespace \rightarrow \statespace$, and the probability distribution $\mu$ over $\noisespace$, from which $\noise_{t}$ is sampled independently at each time step. When the control inputs are determined by a policy $\Policy : \statespace \rightarrow \actionspace$, i.e., $\action_{t} := \Policy(\state_{t})$, $\Sigma$ is said to be \textit{controlled} by $\Policy$.

An infinite sequence of state, action, and disturbance triples $(\state_{t}, \action_{t}, \noise_{t})_{t\in\mathbb{N}_0}$ is a \emph{trajectory} if for all time steps $t\in\mathbb{N}_{>0}$, it holds that $\state_{t+1} = f(\state_{t}, \action_{t}, \noise_{t})$, $\action_t = \Policy(\state_{t})$, and $\noise_{t} \in \support (\mu)$. 
Given a fixed initial state $\state_0$, Markov decision process (MDP) semantics establish a probability space over all such trajectories \protect\cite{puterman_markov_2014}. The probability measure and expectation within this probability space are denoted $\ProbabilityMeasure[\Sigma,\Policy]$ and $\Expectation[\Sigma,\Policy]$, respectively.

For the system semantics and probability measures to be well-defined, we assume $\statespace$, $\actionspace$, and $\noisespace$ to be Borel-measurable. Moreover, we follow the standard assumptions from control theory that $f$ and $\Policy$ are Lipschitz continuous, and that $\statespace$ and $\noisespace$ are closed and bounded.

\subsection{Reach-avoid Verification}
A common way to quantify the performance and safety of control policies for stochastic dynamical systems is via probabilistic reach-avoid specifications, which enforce a threshold on the probability of reaching a target state without encountering unsafe states.
\begin{definition}\label{def:reach-avoid-specification}
    A \emph{reach-avoid specification} is a tuple $(\reach,\avoid,\initial,\Bound)$, where $\reach, \avoid, \initial \subseteq \statespace$ are target states, unsafe states, and initial states, respectively, and $\Bound \in [0,1]$ is a probability threshold, i.e., minimally acceptable probability with which the specification holds. A system $\Sigma$ controlled by a policy $\Policy$ satisfies a reach-avoid specification, denoted $\Sigma,\Policy \vDash (\reach,\avoid,\initial,\Bound)$, if and only if, for all $\state_0\in\initial$,
    \[\ProbabilityMeasure[\Sigma,\Policy]\{\exists\tau\in\mathbb{N}_0 : \state_{\tau}\in\reach \land (\forall t < \tau\in \mathbb{N}_0 : \state_{t}\notin\avoid)\} \geq \Bound.\]
\end{definition}

\begin{example}\label{ex:nine-rooms-prelim-monolithic}
    Consider the original system described in \cref{ex:VeRecycle-examples}, i.e., before water was spilled. 
    This stochastic dynamical system $\Sigma$ has a two-dimensional state space $\statespace = [0, 3]^2$, the two features being x and y coordinates, respectively. The navigation task of safely moving from room 4 to 5 (yellow and green in \cref{fig:VeRecycle-examples-certificate} respectively) with at least a probability of $0.88$, is a reach-avoid specification $\varphi$ where $\reach = [2.4,2.6] \times [1.4, 1.6]$, $\avoid$ is the union of the walls, $\initial=[1.4,1.6]^2 $, and $\rho=0.88$. 
\end{example}

To verify that a reach-avoid specification holds for a system under a particular policy, one can construct probabilistic reach-avoid certificates, whose existence serves as proof. 

\begin{definition}[\cite{zikelic_learning_2023}]\label{def:reach-avoid-certificate}
A continuous function $\Certificate:\statespace\rightarrow\mathbb{R}_{\geq 0}$ is a \emph{reach-avoid certificate} for a system $\Sigma$, policy $\Policy$, and reach-avoid specification $\varphi=(\reach,\avoid,\initial,\Bound)$, if and only if the following conditions hold:
\begin{enumerate}
    \item Initial condition: for all $\state\in\initial$, $\Certificate(\state)\leq 1$;
    \item Safety condition: for all $\state\in\avoid$, $\Certificate(\state)\geq \frac{1}{1-\Bound}$;
    \item Decrease condition: there exists $\varepsilon > 0,$ such that for all $\state\in\statespace\backslash\reach$ either $\mathbb{E}_{\noise\sim\mu}[\Certificate(\state) - \Certificate(f(\state, \Policy(\state), \noise))]\geq \varepsilon$ or $\Certificate(\state)\geq \frac{1}{1-\Bound}$. 
\end{enumerate}
The existence of a probabilistic reach-avoid certificate proves (i.e., \emph{certifies}) that $\System,\Policy\vDash\Specification$. We use the shorthand $\Certificate(S):=\{C(\state)\mid \state\in S\}$ for the image of a subset $S\subseteq\statespace$ under $\Certificate$.
\end{definition}

The z axis (color gradient) in \cref{fig:VeRecycle-examples-certificate} shows an example of a reach-avoid certificate for the specification $\varphi$ and system $\Sigma$ defined in \cref{ex:nine-rooms-prelim-monolithic}. 
Intuitively, the third condition requires that, in any state, $\Sigma$ is unlikely to transition to a state with a higher value in $C$, with larger differences being increasingly improbable. Consequently, the first and second conditions establish that from the initial states (with values $\leq 1$), it is rare to reach the unsafe states (with values $\geq \frac{1}{1-\Bound}$). Furthermore, $\Sigma$ is likely to progress toward the lowest-valued part of the certificate. Since only the target area is permitted to violate the decrease condition and fall below the safety value $\frac{1}{1-\Bound}$, the certificate takes its lowest value in the target area, to which the state of the system eventually converges.

Constructing probabilistic reach-avoid certificates analytically is challenging, especially for nonlinear system dynamics or neural policies. Thus, the state-of-the-art is a counter-example guided synthesis procedure that outputs a certificate in the form of a neural network---a \emph{neural certificate}~\cite{zikelic_learning_2023,chatterjee2023,badings_learning-based_2024}. The procedure loops between learning and verification until a correct certificate is found or the allotted time is exceeded; in the latter case, the result is inconclusive. 

The learning procedure optimizes the candidate certificate on differentiable versions of conditions 1--3 from \cref{def:reach-avoid-certificate} for a set of randomly sampled states and a set of counterexamples produced by the verification procedure. The verification procedure discretizes the continuous state space into ``cells'', verifying conditions 1--3 soundly for all states within the cells using Lipschitz constants and interval bound propagation (IBP)~\cite{mirman_differentiable_2018,gowal_effectiveness_2019}. If the verification is inconclusive for a cell, i.e., when the bounds given by IBP and Lipschitz reasoning are not tight enough to prove or disprove the conditions for all states in the cell, the cell is split into smaller cells to attempt verification again. If a condition is certainly violated for a cell, the verification procedure feeds all states in the cell back into the learning procedure as counterexamples for further optimization.
\section{Problem Statement}\label{sec:problem_statement}
We formulate the general problem of reclaiming guarantees from probabilistic reach-avoid certificates after a change in system dynamics. 
Our problem formulation specifically addresses changes that affect the system dynamics only in a known subset of states $\ChangedSet$. Importantly, this definition does not require additional knowledge of the new dynamics. 
Instead, the aim is to determine guarantees that are valid for any alternative dynamics function, as long as anywhere outside $\ChangedSet$, the function is equivalent to the original dynamics. 

This formulation captures a wide range of changes, such as the detection of unidentified obstacles, the appearance of materials with unknown friction coefficients within a specific region, or the onset of unexplained stochastic phenomena occurring exclusively within a certain velocity range. 

Formally, we consider a stochastic dynamical system $\Sigma = (\statespace, \actionspace,\noisespace,f,\mu)$ controlled by a policy $\Policy$, that satisfies a reach-avoid specification $\varphi = (\reach,\avoid,\initial,\Bound)$, and a reach-avoid certificate $\Certificate$ proving that relation, i.e., $\Certificate$ certifies $\Sigma,\Policy\vDash \varphi$.

\begin{definition}\label{def:reclaimable_threshold}
    Given states $\ChangedSet\subset\statespace$, a threshold $\Bound'\in[0,1]$ is \emph{reclaimable}, 
    if and only if $\Certificate$ certifies $\ChangedSystem, \Policy \vDash (\reach,\avoid,\initial,\Bound')$ for all $\ChangedSystem = (\statespace, \actionspace,\noisespace,\ChangedDynamics,\mu)$, such that
    \begin{equation}\label{eq:changed_set}
    \begin{aligned}
    \{\state\in\statespace \mid \exists \action,\! \noise:\:
    & \ChangedDynamics(\state, \action, \noise) \!\neq\! f(\state, \action, \noise)\} \subseteq \ChangedSet.
    \end{aligned} 
    \end{equation}
\end{definition}

\begin{problem}\label{prob:monolithic}
    Given a certificate $C$, original threshold $\Bound$, and states $\ChangedSet\subset\statespace$, determine the maximum reclaimable threshold $\rho'\leq\Bound$, denoted $\ChangedBound$. 
\end{problem}

\begin{example}[Example~\ref{ex:nine-rooms-prelim-monolithic} continued]\label{ex:nine-rooms-monolithic-problem-statement}
    The spilled water in the blue hatched area in \cref{fig:VeRecycle-examples-certificate} creates a new, alternative system with dynamics function \smash{$\ChangedDynamics$},
    \begin{equation}
    \ChangedDynamics(\state,\action{},\noise{}) = \begin{cases}
        ? & \text{if }\state\in [2,3]\times[0,1]
        \\
        f(\state,\action{},\noise{})& \text{otherwise.}
    \end{cases}
    \end{equation}
    The solution to \cref{prob:monolithic} for $\ChangedSet = [2,3]\times[0,1]$ gives us a threshold $\ChangedBound$ on the probability of reaching green without colliding with walls. Whether the consequences of the spilled water are erratic movements, a complete shut-down, or barely noticeable, should not matter for this threshold: $C$ should certify $\ChangedBound$ for any possible realization of \smash{$\ChangedDynamics$} and be the largest threshold $\leq \rho$ for which this holds.
\end{example}

We remark that \cref{prob:monolithic} defines a maximum threshold that is \textit{always} valid, i.e., without any assumptions on the new dynamics function beyond Equation~\ref{eq:changed_set}, and for a fixed $\Certificate$ and $\Policy$. While higher probability thresholds might be obtainable once the actual new dynamics are specified, such thresholds are not the generally valid thresholds of \cref{prob:monolithic}: we are specifically interested in solving the problem for unknown dynamics in a known set of states. Similarly, updating the certificate or policy for the changed dynamics could also result in a higher threshold than $\ChangedBound$. However, since the state-of-the-art certification of nonlinear, stochastic dynamics and/or neural policies is done with neural networks, changing them requires costly additional training and verification. 

Thus, for applications where a new model of the dynamics can be obtained and alternative policies and certificates can be trained, the general solution to \cref{prob:monolithic} still serves a practical purpose. If the general threshold already matches the original specification or falls within some acceptable margin, there is no need to update the system model and repeat costly re-certification. Thus, a solution to \cref{prob:monolithic} serves as proof that such updates are truly necessary before spending significant computational resources.
\section{VeRecycle}
\label{sec:methodology}
We propose VeRecycle, a framework to reclaim guarantees from reach-avoid certificates given a subset of states $\ChangedSet$ in which dynamics may have changed, as formulated in \cref{prob:monolithic}. VeRecycle exploits a relation between the maximum reclaimable probability threshold $\ChangedBound$ and the infimum of a certificate on $\ChangedSet$. First, we formalize this relation and prove that it provides an exact solution to \cref{prob:monolithic}. Second, we discuss how VeRecycle reduces computational costs through approximations of infima for neural certificates while still providing soundness guarantees. 

\subsection{Exact Solution}\label{sec:exact-solution}
In this subsection, we prove that an exact solution $\ChangedBound$ to \cref{prob:monolithic} depends only on the original bound $\Bound$ and the lowest value the certificate takes for states in $\ChangedSet$: the infimum $\InfimumChangedSet$. 
In particular, we show that whenever $\InfimumChangedSet \geq 1$, $\ChangedBound$ equals
\begin{equation}\label{eq:complete_formula}
     \min\left(\Bound,1-\frac{1}{\InfimumChangedSet}\right),
\end{equation}
and that otherwise no reclaimable threshold exists.

We first show that whenever $\InfimumChangedSet \geq 1$, Equation~\ref{eq:complete_formula} upper-bounds and lower-bounds $\ChangedBound$ (\cref{thm:upper-bound,thm:lower-bound} respectively). Subsequently, we combine these results with a proof of the non-existence of $\ChangedBound$ in case $\InfimumChangedSet < 1$, completing our proof of the relation.

\def\lowerbound{
        If $\InfimumChangedSet \geq 1$, then Equation~\ref{eq:complete_formula} lower-bounds the maximum reclaimable threshold $\ChangedBound$ defined in \cref{prob:monolithic}.
    }
\begin{lemma}[Lower bound of $\ChangedBound$]\label{thm:lower-bound}
    \lowerbound
\end{lemma}
\textit{Proof sketch (proof in App.~\ref{sec:app-lemma-1}).} We set $\rho'$ to Equation~\ref{eq:complete_formula} and prove that it is a reclaimable threshold for $\Certificate$. By definition of a maximum, it follows that the maximum of all such reclaimable thresholds, $\ChangedBound$, is lower-bounded by Equation~\ref{eq:complete_formula}.

Threshold $\rho'$ is reclaimable because it is chosen precisely such that for all states in $\ChangedSet$, the certificate value is at least $\frac{1}{1-\rho'}$. Intuitively, this means that the states with unknown dynamics are visited rarely enough that we do not need to have an expected decrease of the certificate's value there: we reach those states with a probability $\leq 1-\rho'$.

For states outside of the changed states $\ChangedSet$, the expected value of the subsequent states remains unchanged since the dynamics function in $\statespace\backslash \ChangedSet$ and the certificates' values remain unchanged. Note that even if one of the future subsequent states \textit{is} within $\ChangedSet$ and may therefore have new dynamics, this does not change the expectation in the current state, since the expectation is calculated over one time-step only. Conditions 1 and 2 follow similarly because they were satisfied for the original system and threshold, which concludes the proof of \cref{thm:lower-bound}.

\def\upperbound{If $\InfimumChangedSet \geq 1$, then Equation~\ref{eq:complete_formula} upper-bounds the maximum reclaimable threshold $\ChangedBound$ defined in \cref{prob:monolithic}.}
\begin{lemma}[Upper bound of $\ChangedBound$]\label{thm:upper-bound}
    \upperbound
\end{lemma}

\textit{Proof sketch (proof in App.~\ref{sec:app-lemma-2}).} Any reclaimable threshold, including the maximum ($\ChangedBound$), must be a valid threshold for all realizations of \smash{$\ChangedDynamics$}, including fully absorbing dynamics, i.e., where \smash{$\ChangedDynamics(\state,\action,\noise) = \state$}, $\forall\state\in\ChangedSet$. For such absorbing dynamics, it is impossible to prove the expected decrease of the certificate, since the expectation always equals the certificate's value in the current state. Thus, by condition 3 of \cref{def:reach-avoid-certificate}, the infimum of the certificate in $\ChangedSet$ must be above $\frac{1}{1-\ChangedBound}$, from which we derive the inequality in Lemma~\ref{thm:upper-bound}. 

\def\existencesolution{If $\InfimumChangedSet < 1$, then $\ChangedBound$ does not exist.}
\begin{lemma}\label{thm:existence-solution}
    \existencesolution
\end{lemma}
\textit{Proof sketch (proof in App.~\ref{sec:app-lemma-3}).} We prove the implication of \cref{thm:existence-solution} by assuming that $\ChangedBound$ exists, and concluding that $\InfimumChangedSet\geq 1$. If $\ChangedBound$ exists, it must be valid for fully absorbing dynamics: \smash{$\ChangedDynamics(\state,\action,\noise) = \state$}, $\forall\state\in\ChangedSet$. Thus, by condition 3 of \cref{def:reach-avoid-certificate}, the infimum of the certificate in $\ChangedSet$ must be above $\frac{1}{1-\ChangedBound}$, from which we derive $\InfimumChangedSet\geq 1$. 

\def\equality{The maximum reclaimable threshold $\ChangedBound$ defined in \cref{prob:monolithic} does not exist if $\InfimumChangedSet < 1$, and otherwise $\ChangedBound$ equals Equation~\ref{eq:complete_formula}.}
\begin{theorem}[Equality to $\ChangedBound$]\label{thm:equality}
    \equality
\end{theorem}
\Cref{thm:equality} follows directly from \cref{thm:lower-bound,thm:upper-bound,thm:existence-solution}. A formal proof is given in App.~\ref{sec:app-theorem-1}.

\subsection{Sound Approximation}\label{sec:sound-approximation}
\Cref{thm:equality} proves that the exact solution to \cref{prob:monolithic} only depends on the infimum of the certificate for $\ChangedSet$ and the original certified threshold $\Bound$. In practice, however, finding the infimum of a certificate can be expensive and difficult to scale, particularly for neural certificates. To address this, VeRecycle is equipped to work with approximations of the infimum, while still producing only sound thresholds. 

Specifically, methods such as interval bound propagation (IBP)~\cite{gowal_effectiveness_2019} can be used to determine values $I^-$ and $I^+$ such that $I^- \leq \InfimumChangedSet \leq I^+$. With these bounds on the infimum, we can determine guaranteed bounds on the exact solution in Equation~\ref{eq:complete_formula}:
\begin{equation}\label{eq:threshold-bounds}
    \ChangedBound^-=\min\left(\Bound, 1 - \frac{1}{I^-}\right) ; \quad \ChangedBound^+ =\min\left(\Bound, 1 - \frac{1}{I^+}\right).
\end{equation}

\def\approximation{If $I^- \leq \InfimumChangedSet \leq I^+$ and $I^-\geq 1$, then $\ChangedBound^-\leq \ChangedBound \leq \ChangedBound^+$.}
\begin{theorem}\label{thm:approximation}
    \approximation
\end{theorem}
\textit{Proof sketch (proof in App.~\ref{sec:app-theorem-2}).} We show that the function in Equation~\ref{eq:complete_formula} is monotonically increasing for $\InfimumChangedSet$ on the interval $(0,\infty)$. Additionally, if $\InfimumChangedSet\geq 1$ the formula in Equation~\ref{eq:complete_formula} is, by \cref{thm:equality}, equal to $\ChangedBound$. Thus, if $1 \leq I^- \leq \InfimumChangedSet \leq I^+$, it follows that $\ChangedBound^-\leq \ChangedBound \leq \ChangedBound^+$. 

VeRecycle outputs $\ChangedBound^-$ as a guaranteed threshold, certified by the original certificate for any realization of \smash{$\ChangedDynamics$}. The tightness of $I^-$ to $\InfimumChangedSet$---and, consequently, $\ChangedBound^-$ to $\ChangedBound$---depends on the size of the interval used for IBP. Thus, VeRecycle uses a mesh to split up $\ChangedSet$ into smaller intervals, increasing the tightness of $\ChangedBound^-$ to $\ChangedBound$. 

If $\ChangedBound^-$ is unsatisfactory, e.g., it does not match the original guarantee, or exceeds some permissible drop in probability, $\ChangedSet$ can be split up using a finer mesh. Such a refinement is only done if the desired threshold is below $\ChangedBound^+$: otherwise, the desired threshold cannot be obtained by refinement. In those cases, VeRecycle's output serves as proof that certification of the specification cannot be achieved without making additional assumptions on the changed dynamics, and/or performing additional training of the neural certificate. 

The complexity of VeRecycle depends only on the complexity of determining (approximate) infima. For neural certificates, it is therefore determined by the complexity of IBP and the mesh size used.
\section{VeRecycle for Compositional Control}\label{sec:compositional}
VeRecycle is particularly valuable for systems under compositional control, in which a selection of policies is executed sequentially to complete a complex task. The modularity of such control allows VeRecycle to first reclaim guarantees for each component individually. Subsequently, the least affected components can be reused directly in a new composition, minimizing the need for costly repairs to neural control policies, certificates, and system models.

\begin{figure}[t]
        \centering
        \includegraphics[width=\columnwidth]{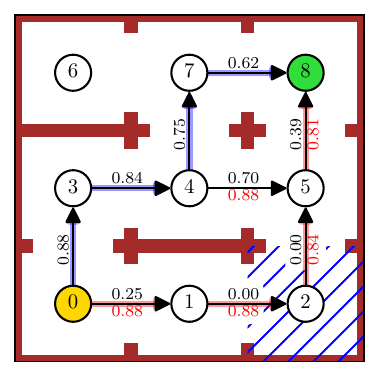}
        \caption{A compositional policy (blue) for a decomposed reach-avoid task: $0$ (yellow) $\rightarrow 8$ (green). 
        Vertices are discrete abstractions of the rooms in the continuous environment. A directed edge signifies the existence of a certified policy between two rooms. Updated thresholds (black; original in red) reveal that the old path (red edges) no longer maximizes the global probability threshold.}\label{fig:VeRecycle-examples-graph}
\end{figure}

We consider compositional control based on a reach-avoid specification that is decomposed into several connected subtasks, either manually or automatically~\cite{jothimurugan2021}. This task decomposition is represented as a directed graph $G=(V, E, \InitialVertex, \TargetVertex, \beta)$ where the mapping function $\beta:V\cup E\rightarrow\statespace$ relates the graph's elements to a stochastic dynamical system's state space: an edge $e$ from vertex $v_a$ to $v_b$ represents the task of reaching $\beta(v_b)$ from $\beta(v_a)$ while avoiding $\beta(e)$. 

A compositional policy for $G$ then consists of a set of edge-associated policies $\{\Policy[e]\mid e\in E\}$ and a path $\GlobalPolicy = (\InitialVertex, \dots, v_k)$ in $G$, representing sequential execution of the policies $\Policy[e]$ associated with each traversed edge $e$. Certification of the global reach-avoid specification $\Specification$ with threshold $\Bound$, requires that valid probability thresholds $\Bound[e]$ and corresponding certificates $\Certificate[e]$ are given for all edge-associated policies $\Policy[e]$, and that multiplying them along the edges of the path $\GlobalPolicy$ results in a value above the global probability threshold, i.e. $\prod_{i = 1}^{k} \Bound[(v_{i-1}, v_{i})] \leq \Bound$.

\begin{example}\label{ex:nine-rooms-compositional}
    We return to the running example of a robot navigating a warehouse with nine rooms. The goal---to navigate from the center of room 0 to the center of room 8---is broken down into subtasks of moving between two specific rooms. The specific task decomposition graph $G$ is shown in \cref{fig:VeRecycle-examples-graph}, with function $\beta$ mapping vertices 0--8 to the dashed regions in \cref{fig:VeRecycle-examples-certificate}, and each edge to the walls ($\avoid$). The compositional policy used to control the robot consists of the path $\GlobalPolicy = (0,1,2,5,8)$ highlighted in red, and for each edge $e$, a policy $\Policy[e]$. This compositional policy is certified for the original dynamics of the system: for each edge $e$, some certificate $\Certificate[e]$ is given that proves a threshold $\Bound[e]$.
\end{example}

\Cref{alg:compositional_algorithm} shows how VeRecycle can be used to reclaim guarantees for such compositional policies, given the set $\ChangedSet$ with changed dynamics, the graph $G$, and the sets of edge-associated certificates. 
For each edge in the task decomposition graph $G$ used by the compositional controller, VeRecycle is invoked to reclaim a guaranteed threshold for the original certificate, ensuring validity for any change within $\ChangedSet$ (\cref{line:verecycle_call}). 
Each edge in $G$ is assigned the weight $-\log(p)$, where $p$ is the threshold obtained from VeRecycle (\cref{line:assign_weight}). The problem of maximizing the probability threshold of a path becomes a minimization problem that can be solved by an out-of-the-box shortest-path solver (\cref{line:shortest_path}). The sum of the weights along the shortest path is then converted into the probability threshold for that path (\cref{line:cost_to_bound}).

\begin{example}[Example~\ref{ex:nine-rooms-compositional} continued]\label{ex:changed-nine-rooms-compositional}
    Once water has been spilled, the new dynamics \smash{$\ChangedDynamics$}---unknown in the states $\ChangedSet = [2,3]\times[0,1]$---take effect. 
    Given $\ChangedSet, G, \{\Certificate[e]\mid e\in E\}$, and $ \{\Bound[e] \mid e\in E\}$, \cref{alg:compositional_algorithm}---invoking VeRecycle for each edge---outputs an alternative blue path $\ChangedGlobalPolicy=(0,3,4,7,8)$ highlighted in \cref{fig:VeRecycle-examples-graph} which is certified to have at least $\ChangedBound = 0.88 \cdot 0.84 \cdot 0.75\cdot 0.62 \geq 0.33$, regardless of the effect of the spilled water. 
    Although the certified threshold on the original path $\GlobalPolicy$ has dropped to zero, an alternative with non-trivial guarantees could still be obtained. This illustrates how VeRecycle can exploit the modular nature of compositional policies while leaving policies and certificates (often neural networks) unaltered.
\end{example}

The complexity of \cref{alg:compositional_algorithm} depends on the complexity of the calls to VeRecycle in \cref{line:verecycle_call}, made once for each edge in the graph. While this complexity could likely be reduced further by minimizing the number of VeRecycle calls, e.g., a shortest-path solver with on-demand calls to VeRecycle only, this is outside the scope of this work. \Cref{alg:compositional_algorithm} serves as a proof-of-concept implementation of VeRecycle applied to compositional control. 

\begin{algorithm}[tb]
    \caption{Recomposing Compositional Policies}
    \label{alg:compositional_algorithm}
    \KwData{Set $\ChangedSet$, graph $G=(V,E,\InitialVertex,\TargetVertex,\beta)$, certificates $\{\Certificate[e] \mid e\in E\}$ and thresholds $\{\Bound[e] \mid e\in E\}$.}
    
    \KwResult{Path $\ChangedPath$ and threshold $\ChangedBound$.}
    $W \leftarrow$ empty hash map\\
    \ForEach{$e \in E$}{
            $\ChangedBound[e] \leftarrow \textsc{VeRecycle}(\Certificate[e], \Bound[e], \ChangedSet)$\\ \label{line:verecycle_call}
            $W[e] \leftarrow -\log{\ChangedBound[e]}$\\ \label{line:assign_weight}
    }
    $(\InitialVertex, \dots, \TargetVertex), $ weight $ \leftarrow \textsc{ShortestPath} (G,W)$\\\label{line:shortest_path}
    $\ChangedBound \leftarrow 2^{-\text{weight}}$\\\label{line:cost_to_bound}
    \Return $(v_0, \dots, v_{\star}),\ChangedBound$\\
\end{algorithm}

\section{Evaluation}
\label{sec:evaluation}

We demonstrate
a VeRecycle implementation\footnote{Code at https://github.com/SUMI-lab/VeRecycle} with interval bound propagation~\cite{gowal_effectiveness_2019} to approximate infima of neural certificates. 
We analyze the efficiency and quality of reclaimed guarantees by VeRecycle on both stand-alone and compositional neural control in the ``Stochastic Nine Rooms'' benchmark (\cref{fig:VeRecycle-examples-graph}) introduced by \citeauthor{zikelic_compositional_2023}~[\citeyear{zikelic_compositional_2023}], with new, unknown dynamics in one of the rooms (e.g., the blue hatched area).

Out-of-the-box baselines are unavailable since, to the best of our knowledge, VeRecycle is the first framework to reclaim sound guarantees without requiring updated models of the dynamics. As a baseline, we therefore use state-of-the-art neural certification~\cite{badings_learning-based_2024} on worst-case dynamics in $\ChangedSet$ (absorbing, i.e., $\state_{t+1} = \state_{t}$), treating $\ChangedSet$ as part of the unsafe set $\avoid$. We consider three settings for this baseline. First (B-I), verifying that a given new threshold---output by VeRecycle---is valid for the original certificate, i.e., given $C$ and $\rho'$, decide correctness. Second (B-II), finding a correct threshold given the original certificate as a candidate, i.e., given $C$, find a valid $\rho'$ and $C'$. Third (B-III), finding a correct threshold and certificate from scratch, i.e., determine $\rho'$ and $C'$, without $C$. 
All experiments are run on
\citeauthor{DHPC2024} [\citeyear{DHPC2024}] with NVIDIA A100 GPUs.

\subsection{Experiment 1: Individual Certificates}
We analyze VeRecycle's reclaimed thresholds for 63 unique tasks, combining 9 reach-avoid specifications (graph edges in \cref{fig:VeRecycle-examples-graph}) with 7 different subsets $\ChangedSet$ (rooms). As input, we use neural policies obtained through proximal policy optimization (PPO)~\cite{schulman_proximal_2017} with neural reach-avoid certificates~\cite{badings_learning-based_2024}. Technical details are provided in App.~\ref{app:experimental_setup}.
\begin{table}[t]
\centering

\begin{tabular}{@{}l r@{\hspace{6pt}}r@{\hspace{6pt}}r@{\hspace{6pt}}r}
\toprule
& & \multicolumn{3}{c}{Threshold $\ChangedBound^-$}\\
\cmidrule(lr){3-5} 
Method & Time (s) & Loose $C$ & Tight $C$ & All \\
\midrule
VeRecycle & \textbf{0.33} \!\tiny{$\pm$\! 0.66}     &  0.58         \!\tiny{$\pm$\! 0.35} & \textbf{0.65} \!\tiny{$\pm$\! 0.35} & 0.59      \!\tiny{$\pm$\! 0.35} \\
B-I & 9.22                \!\tiny{$\pm$\! 2.79}     & 0.58          \!\tiny{$\pm$\! 0.35} & \textbf{0.65} \!\tiny{$\pm$\! 0.35} & 0.59      \!\tiny{$\pm$\! 0.35} \\
B-II 
& 191.73 \!\tiny{$\pm$\! 161.64}   
& \textbf{0.61} \!\tiny{$\pm$\! 0.34} 
& 0.64 \!\tiny{$\pm$\! 0.34} 
& \textbf{0.62} \!\tiny{$\pm$\! 0.34} \\
B-III 
& 918.96 \!\tiny{$\pm$\! 190.64}   
& 0.51   \!\tiny{$\pm$\! 0.31} 
& 0.53   \!\tiny{$\pm$\! 0.29} 
& 0.51   \!\tiny{$\pm$\! 0.31} \\
\bottomrule\end{tabular}
\caption{Runtime and reclaimed probability thresholds, averaged (with standard deviation) over 5 random seeds and 63 tasks (9 specifications $\times$ 7 changes). Thresholds are grouped by the quality of original certificate $C$: tight certificates (available for 2 out of 9 specifications) are below $\frac{1}{1-\Bound}$ only inside the rooms containing $\initial$ and $\reach$. Best result per column is marked in bold.}
\label{tab:experiment-monolithic}
\end{table}

In runtime (\cref{tab:experiment-monolithic}, column 2), VeRecycle outperforms the baseline on all settings. 
Searching for a new threshold and certificate from scratch (B-III) requires the highest computation time. Using the original certificate as a candidate (B-II) reduces computation time by an order of magnitude, emphasizing the benefit of reusing original verification results. When informed of a correct threshold by VeRecycle (B-I) runtime decreases further, demonstrating the benefit of inferring thresholds from certificates rather than trial-and-error search. VeRecycle outperforms B-I by two orders of magnitude, showing the efficiency of ignoring unaffected states.

In addition to requiring less computational effort, VeRecycle produces tighter probability thresholds (\cref{tab:experiment-monolithic}, columns 3--5) than re-verification from scratch (B-III). While re-verification with the original certificate (B-II) produces higher thresholds on average, VeRecycle performs comparably on high-quality input certificates---where $\Certificate(\state) \geq \frac{1}{1-\Bound}$ for all rooms outside the initial and target rooms. As illustrated in \cref{fig:certificate_comparison}, B-II can refine unnecessary low-valued states in ``loose'' certificates to achieve a higher threshold than VeRecycle. However, this refinement process is computationally expensive, making VeRecycle a lightweight alternative to identify when certificate refinement is necessary. Future work on improving certification techniques for generating higher-quality certificates would further enhance VeRecycle’s performance and applicability.

\begin{figure}
    \centering
    \includegraphics[width=\linewidth]{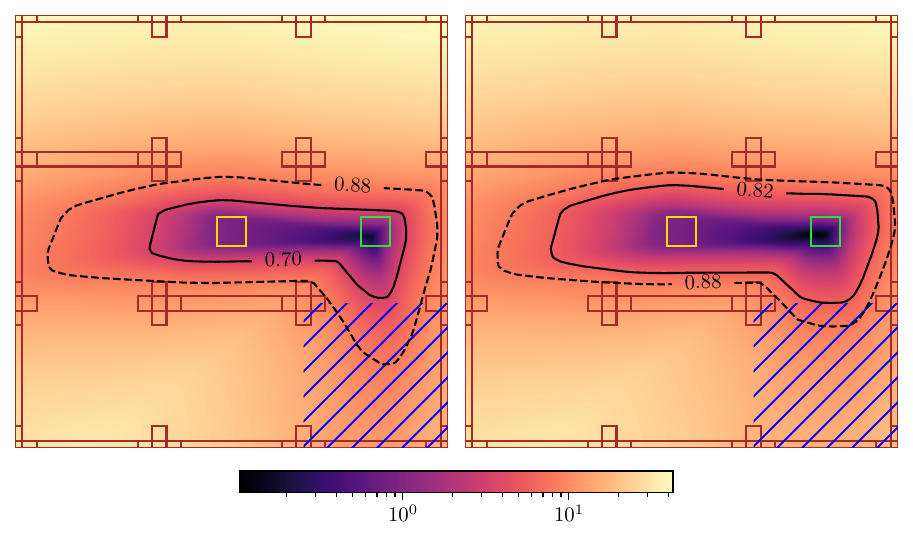}
    \caption{Example of a ``loose'' original certificate (gradient left) refined by baseline algorithm (gradient right) to obtain an updated threshold $\ChangedBound$ (solid line) closer to the original $\Bound$ (dashed line). Yellow ($\initial$), green ($\reach$), red ($\avoid$), and blue ($\ChangedSet$) denote task sets.}
    \label{fig:certificate_comparison}
\end{figure}

\subsection{Experiment 2: Compositional Control}

\begin{table}[t]
\centering
\begin{tabular}{lccc}
\toprule
& \multicolumn{3}{c}{Avg. $\ChangedBound^-$ per $\ChangedSet$ location} \\
\cmidrule(lr){2-4} 
Method & Target adj. & Other & All \\
\midrule
VeRecycle   & 0.24          \!\tiny{$\pm$\! 0.09} & \textbf{0.44} \!\tiny{$\pm$\! 0.06} & 0.38 \!\tiny{$\pm$\! 0.12} \\
B-I         &0.24           \!\tiny{$\pm$\! 0.09} &\textbf{0.44}  \!\tiny{$\pm$\! 0.06} &0.38 \!\tiny{$\pm$\! 0.12} \\
B-II        & \textbf{0.40} \!\tiny{$\pm$\! 0.07} & 0.42          \!\tiny{$\pm$\! 0.06} &\textbf{0.42} \!\tiny{$\pm$\! 0.06} \\
B-III       & 0.19 \!\tiny{$\pm$\! 0.09} & 0.21          \!\tiny{$\pm$\! 0.06} & 0.20 \!\tiny{$\pm$\! 0.07} \\
\bottomrule\end{tabular}
\caption{Reclaimed probability bounds for compositional policy (Experiment 2), averaged over 5 random seeds for 7 different sets $\ChangedSet$ with new dynamics. Grouped by whether $\ChangedSet$ is adjacent to the target room, i.e., in room 5 or 7. Best results per column are in bold.}
\label{tab:global_bounds}
\end{table}
Next, we compare the global thresholds obtained with \cref{alg:compositional_algorithm} on the graph from \cref{fig:VeRecycle-examples-graph}, using VeRecycle and baselines to reclaim guarantees for each graph edge, and NetworkX's Dijkstra~\cite{hagberg_exploring_2008} as shortest-path solver. The results (\cref{tab:global_bounds}) show that VeRecycle performs competitively on the majority of tasks. 

In two particular tasks---rooms with vertex 5 and 7 having unknown dynamics---the effect of low-quality certificates on VeRecycle's performance can still be observed. In those cases, the edges (7,8) and (5,8)---both with ``loose'' certificates---are simultaneously impacted, thus impacting the reclaimed guarantee on all feasible paths to the target states. 

Overall, we observe that VeRecycle is particularly powerful in compositional control settings, because the impact of low-quality input certificates is mitigated by the availability of alternative policies. 

\section{Related Work}
\label{sec:related_work}

\paragraph{Probabilistic neural reach-avoid certification.}
Probabilistic neural certification is the state-of-the-art for verifying reach-avoid specifications for stochastic dynamical systems under neural control~\cite{zikelic_learning_2023,ansaripour2023,chatterjee2023,mathiesen2023,abate2024stochastic,badings_learning-based_2024}. 
VeRecycle complements these methods by efficiently deciding whether such guarantees are maintained when system dynamics change, minimizing the high computational cost of re-certifying the complete updated system model. 

\paragraph{Robust control synthesis.}
Robust control synthesis addresses changes in dynamics by designing control policies that perform reliably under a modeled range of disturbances. 
Techniques include adaptations of Markov Decision Process (MDP) solvers for richer models expressing multiple variations of the system: e.g., parametric MDPs \cite{badings_scenario-based_2022}, interval MDPs \cite{jiang2022,coppola_data-driven_2024}, uncertain MDPs \cite{wolff_robust_2012}, and multi-environment MDPs \cite{raskin_multiple-environment_2014}. 
While the type of changes addressed by VeRecycle could theoretically be addressed by robust synthesis techniques as well, this would require either knowing all possible changes in advance or updating and solving the adapted MDP again. Thus, we see VeRecycle as a lightweight alternative that does not require prior knowledge of possible changes.

\paragraph{Reusing policies in compositional control.}
Compositional control as described in \cref{sec:compositional} stems from the concept of options~\cite{sutton_between_1999}, adapted for safety-critical applications through correct-by-design synthesis methods for a high-level plan \cite{jothimurugan2021,neary_verifiable_2023} and formal verification of used components \cite{ivanov_compositional_2021,delgrange2024controller}. Notably, \cite{zikelic_compositional_2023} verify each component with probabilistic neural certification, inspiring our application of VeRecycle to compositional control. In addition to solving more difficult, long-horizon tasks, a motivation for compositional control is reusing the components in alternative high-level plans. This is not straightforward in case of changes in system dynamics, however, since guarantees on components must still be reclaimed before composing a (new) high-level plan, making VeRecycle complementary to compositional control.
\section{Conclusion}
\label{sec:conclusion}
VeRecycle is the first framework to formally reclaim guarantees from probabilistic neural certificates for discrete-time stochastic dynamical systems after localized changes. 
We formally prove that VeRecycle infers sound guarantees from (approximated) infima of the original certificates. 
Our experimental evaluation demonstrates scenarios in which VeRecycle's guarantees are competitive---particularly for compositional control---compared to complete re-certification, with computational cost reduced by three orders of magnitude. 

\paragraph{Future work.} We plan to extend the presented theoretical foundations to probabilistic neural certificates for continuous-time systems~\cite{neustroev2025neural}. Additionally, we aim to explore how the certificate learning process can be optimized to maximize the guarantees reclaimed with VeRecycle. Another avenue of future research is synergies between VeRecycle and learning-enabled certification~\cite{kolter2019,wu2023neural,takeishi2021,schlaginhaufen2021learning}, using VeRecycle to pinpoint crucial states for which new dynamics must be learned.
\clearpage
\appendix
\setcounter{lemma}{0}
\setcounter{theorem}{0}
\newcommand{\proofsection}[1]{%
    \section{Proof of \cref{thm:#1}}\label{app:proof-#1}%
}
\section*{Technical Appendix}
\proofsection{lower-bound}\label{sec:app-lemma-1}
\begin{lemma}[Restated]
    \lowerbound
\end{lemma}
\begin{proof}
    Choose a system $\Sigma=(\statespace,\actionspace,\noisespace,f,\mu)$, policy $\Policy$, reach-avoid specification $\varphi=(\reach,\avoid,\initial,\Bound)$, and certificate $C$ such that $C$ certifies $\Sigma,\Policy\vDash\varphi$. Choose arbitrary $\ChangedSet\subset\statespace$ such that $\InfimumChangedSet \geq 1$ and $\ChangedDynamics$ such that Equation~\ref{eq:changed_set} is satisfied. Let $\Bound'$ equal Equation~\ref{eq:complete_formula}.

    By \cref{prob:monolithic}, $\ChangedBound$ is the maximum probability in $[0,\Bound]$ that is a threshold certified by $\Certificate$ for the alternative dynamics $\ChangedDynamics$, i.e. $\Certificate$ certifies $\ChangedSystem,\Policy\vDash (\reach,\avoid,\initial,\ChangedBound)$ for $\ChangedSystem = (\statespace,\actionspace,\noisespace,\ChangedDynamics,\mu)$. 
    Thus, to show that $\Bound' \leq \ChangedBound$, we will show that $\Bound'\in[0,\Bound]$ and that $\Certificate$ certifies $\ChangedSystem,\Policy\vDash (\reach,\avoid,\initial,\Bound')$ for $\ChangedSystem = (\statespace,\actionspace,\noisespace,\ChangedDynamics,\mu)$. 
    
    First, note that by definition of a minimum, $\Bound' \leq \Bound$. Since we can derive from $\InfimumChangedSet \geq 1$ that $1-\frac{1}{\InfimumChangedSet}\geq 0$, and that $\Bound\geq 0$ by \cref{def:reach-avoid-specification}, it follows that $\rho'\in[0,\rho]$.

    Next, we show that conditions 1--3 from \cref{def:reach-avoid-certificate} hold for $\Certificate$ w.r.t. $\ChangedSystem$ and $\rho'$. 

    \begin{enumerate}
    \item Since $\Certificate$ certifies $\System$ for $\rho$, by condition 1 of \cref{def:reach-avoid-certificate} it holds for all $\state\in\initial$ that $\Certificate(\state)\leq 1$. Since neither $\initial$, nor $\Certificate$ has changed, condition 1 also holds for $\Certificate$ w.r.t. $\ChangedSystem$ and $\rho'$.

    \item Since $\Certificate$ certifies $\System$ for $\rho$, by condition 2 of \cref{def:reach-avoid-certificate} it holds for all $\state\in\avoid$ that $\Certificate(\state)\geq \frac{1}{1-\Bound}$. Since $\Bound \geq \Bound'$, 
    it follows that for all $\state\in\avoid$, $\Certificate(\state)\geq \frac{1}{1-\Bound}\geq \frac{1}{1-\Bound'}$. Thus, condition 2 holds for $\Certificate$ w.r.t. $\ChangedSystem$ and $\rho'$.
    
    \item Since $\Certificate$ certifies $\System$ for $\rho$, by condition 3 of \cref{def:reach-avoid-certificate} it holds that there exists a $\varepsilon > 0$ such that $\forall\state\in\statespace\backslash\reach$ either $\mathbb{E}_{\noise\sim\mu}[\Certificate(\state) - \Certificate(\Dynamics(\state, \Policy(\state), \noise))]\geq \varepsilon$ or $\Certificate(\state)\geq \frac{1}{1-\Bound}$. 
    Choose $\state\in\statespace\backslash\reach$. By division into cases we will show that either $\mathbb{E}_{\noise\sim\mu}[\Certificate(\state) - \Certificate(\ChangedDynamics(\state, \Policy(\state), \noise))]\geq \varepsilon$, or $\Certificate(\state)\geq \frac{1}{1-\Bound'}$ holds. 
        \begin{itemize}
            \item \textbf{Case 1 ($\Certificate(\state)\geq \frac{1}{1-\Bound}$)}:
            Suppose $\Certificate(\state)\geq \frac{1}{1-\Bound}$. 
            Since $\Bound'\leq \Bound$, it follows that $\Certificate(\state)\geq \frac{1}{1-\Bound}\geq \frac{1}{1-\Bound'}$.

            \item \textbf{Case 2 ($\Certificate(\state)< \frac{1}{1-\Bound}$ and $\state\in\ChangedSet$)}:
            Suppose $\Certificate(\state)< \frac{1}{1-\Bound}$ and $\state\in\ChangedSet$. It holds (by definition of an infimum) that $\Certificate(\state)\geq  {\InfimumChangedSet}$. Since our choice of $\Bound'$ ensures $\Bound' \leq 1-\frac{1}{\InfimumChangedSet}$, we can derive that $\Certificate(\state)\geq  \InfimumChangedSet \geq \frac{1}{1-\Bound'}$.

            \item \textbf{Case 3 ($\Certificate(\state)< \frac{1}{1-\Bound}$ and $\state\notin\ChangedSet$)}:
            Suppose $\Certificate(\state)< \frac{1}{1-\Bound}$ and $\state\notin\ChangedSet$. It follows that $\mathbb{E}_{\noise\sim\mu}[\Certificate(\state) - \Certificate(\Dynamics(\state, \Policy(\state), \noise))]\geq \varepsilon$. Since $\state\in\ChangedSet$ it follows from Equation~\ref{eq:changed_set} that $\Dynamics(\state,\action,\noise)=\ChangedDynamics(\state,\action,\noise)$. Therefore,
            \begin{multline*}
                \mathbb{E}_{\noise\sim\mu}[\Certificate(\state) - \Certificate(\ChangedDynamics(\state, \pi(\state), \noise))] \\ =
                \mathbb{E}_{\noise\sim\mu}[\Certificate(\state) - \Certificate(\Dynamics(\state, \pi(\state), \noise))] \geq \varepsilon.
            \end{multline*}
        \end{itemize}
        
        Since we have proven for all cases that either $\mathbb{E}_{\noise\sim\mu}[\Certificate(\state) - \Certificate(\ChangedDynamics(\state, \pi(\state), \noise))]\geq \varepsilon$, or $\Certificate(\state)\geq \frac{1}{1-\rho'}$, we conclude that condition 3 holds for $C$ w.r.t. $\ChangedSystem$ and $\rho'$. \qedhere
    \end{enumerate}
\end{proof}

\proofsection{upper-bound}\label{sec:app-lemma-2}
\begin{lemma}[Restated]
    \upperbound
\end{lemma}

\begin{proof}
    Suppose $\InfimumChangedSet \geq 1$.
    Let 
    \begin{equation}
    \ChangedDynamics(\state,\action,\noise)= \begin{cases}
        \state & \text{for }\state\in\ChangedSet
        \\
        f(\state,\action,\noise)& \text{otherwise.}
    \end{cases}
    \end{equation}
    Since $\ChangedDynamics$ satisfies \cref{eq:changed_set}, it follows from \cref{prob:monolithic} that $C$ certifies $\ChangedSystem, \Policy\vDash (\reach,\avoid,\initial,\ChangedBound)$ for $\ChangedDynamics$. 
    
    Choose $\state\in\ChangedSet$ such that $\Certificate(x) = \InfimumChangedSet$. 
    Since $\ChangedDynamics(\state,\action,\noise)=\state$, $\mathbb{E}_{\noise\sim\mu}[\Certificate(\state) - \Certificate(\ChangedDynamics(\state, \Policy(\state), \noise))] = \mathbb{E}_{\noise\sim\mu}[\Certificate(\state) - \Certificate(\state)] = 0$. Thus, by condition 3 of \cref{def:reach-avoid-certificate}, it must hold that $\InfimumChangedSet \geq \frac{1}{1-\ChangedBound}$. 
    Rewriting this inequality gives $1-\frac{1}{\InfimumChangedSet}\geq \ChangedBound$. Additionally, by \cref{prob:monolithic}, $\Bound\geq \ChangedBound$. Thus, $\min(\Bound, 1-\frac{1}{\InfimumChangedSet})\geq \ChangedBound$.
\end{proof}

\proofsection{existence-solution}\label{sec:app-lemma-3}
\begin{lemma}[Restated]
    \existencesolution
\end{lemma}
\begin{proof}
    We prove the implication by assuming that $\ChangedBound$ exists, and concluding that $\InfimumChangedSet \geq 1$.
    
    Suppose $\ChangedBound$ exists. Let 
    \begin{equation}
    \ChangedDynamics(\state,\action,\noise)= \begin{cases}
        \state & \text{for }\state\in\ChangedSet
        \\
        f(\state,\action,\noise)& \text{otherwise.}
    \end{cases}
    \end{equation}
    Since $\ChangedDynamics$ satisfies \cref{eq:changed_set}, it follows from \cref{prob:monolithic} that $C$ certifies $\ChangedSystem, \Policy\vDash (\reach,\avoid,\initial,\ChangedBound)$ for $\ChangedDynamics$. 

    Choose $\state\in\ChangedSet$ such that $\Certificate(x) = \InfimumChangedSet$. 
    Since $\ChangedDynamics(\state,\action,\noise)=\state$, we do not have an expected decrease in $\state$: 
    \[\begin{aligned}
        \mathbb{E}_{\noise\sim\mu}[\Certificate(\state) - \Certificate(\ChangedDynamics(\state, \Policy(\state), \noise))] &= \\\mathbb{E}_{\noise\sim\mu}[\Certificate(\state) - \Certificate(\state)] &= 0.
    \end{aligned}\] 
    Thus, by \cref{def:reach-avoid-certificate}, condition 3, it must hold that $\InfimumChangedSet \geq \frac{1}{1-\ChangedBound}$. 
    By \cref{def:reach-avoid-specification}, $\ChangedBound\geq 0$, and, consequently, $\frac{1}{1-\ChangedBound} \geq 1$. Thus, $\InfimumChangedSet \geq \frac{1}{1-\ChangedBound} \geq 1$.
\end{proof}

\proofsection{equality}\label{sec:app-theorem-1}
\begin{theorem}[Restated]
    \equality
\end{theorem}
\begin{proof}
Suppose $\InfimumChangedSet < 1$. By \cref{thm:existence-solution}, it follows that $\ChangedBound$ does not exist. 

Suppose $\InfimumChangedSet \geq 1$. Let $\rho'$ equal Equation~\ref{eq:complete_formula}. By \cref{thm:lower-bound}, $\Bound' \leq\ChangedBound$. By \cref{thm:upper-bound}, $\Bound'\geq\ChangedBound$. Thus, $\Bound' = \ChangedBound$.
\end{proof}

\proofsection{approximation}\label{sec:app-theorem-2}
\begin{theorem}[Restated]
    \approximation
\end{theorem}

\begin{proof}
    The function $f(x) = 1-\frac{1}{x}$ is monotonically increasing on the interval $(0,\infty)$, since the derivative $f'(x)= \frac{1}{x^2}$ is positive on that interval.
    
    The function $g(x) = \min(x,\rho)$ is monotonically increasing, proven by the following (exhaustive) three cases:
    \begin{itemize}
        \item $\rho\leq a \leq b \rightarrow g(a)=\rho\leq \rho=g(b)$,
        \item $a\leq \rho \leq b \rightarrow g(a)=a\leq \rho=g(b)$,
        \item $a\leq b \leq \rho\rightarrow g(a)=a\leq b=g(b)$.
    \end{itemize}
    It follows that the composition $g$ of $f$, \[h(x) = g\circ f(x) = \min(\rho, 1-\frac{1}{x}),\] is also monotonically increasing on the interval $(0,\infty)$.

    Assume $1 \leq I^- \leq \InfimumChangedSet \leq I^+$.
    Since $h$ increases monotonically for $x\in(0,\infty)$, we have
    \[h(I^-) \leq h(\InfimumChangedSet) \leq h(I^+),\]
    which, by Equation~\ref{eq:threshold-bounds}, implies
    \[\ChangedBound^- \leq h(\InfimumChangedSet) \leq \ChangedBound^+.\]
    From \cref{thm:equality} and $\InfimumChangedSet\geq 1$, it follows that $\ChangedBound=h(\InfimumChangedSet)$. Thus,
    \[\ChangedBound^-\leq\ChangedBound\leq\ChangedBound^+.\qedhere\]
\end{proof}
\section{Experiments}\label{app:experimental_setup}
\subsection{Dynamics}
\label{app:dynamics}
The Stochastic Nine Rooms task has $\statespace = [0,3]^2\subset \mathbb{R}^2$, $\actionspace = [-2,2]^2 \subset \mathbb{R}^2$, and $\noisespace = [-0.005, 0.005]^2 \subset \mathbb{R}^2$. 
The dynamics function of the system is 
\begin{equation*}
    f(\state,\action,\noise) = \state + 0.1 \min (\max (\action, -1), 1) + \noise.
\end{equation*}
The disturbance vector $\noise$ has two components, $\noise_1$ and $\noise_2$, independently drawn from a triangular probability distribution on $\noisespace$.
\subsection{Reach-Avoid Objective}
For the global reach-avoid objective, the set of initial states is $\initial = [0.4,0.6]^2$ and the set of target states is $\reach = [2.4,2.6]^2$. The unsafe states $\avoid$ are all the brown walls in \cref{fig:VeRecycle-examples-graph}, where the walls along the border have a thickness of $0.05$ and all other walls have a thickness of $0.1$. 
The task decomposition graph $G$ is the DAG given in \cref{fig:VeRecycle-examples-graph} with $\beta$ such that each vertex $v\in[0,8]$ corresponds to the centers of each of the 9 rooms they are drawn in. More formally, 
\[[0.4 + (v \bmod 3),0.6 + (v \bmod 3)] \times [0.4 + \left\lfloor{\frac{v}{3}}\right\rfloor,0.6 +\left\lfloor{\frac{v}{3}}\right\rfloor].
\]
For each edge $e$, $\beta(e)=\avoid$, the unsafe states of the global task (all walls).

\subsection{Policy and Certificate Synthesis}\label{app:synthesis_original_control}
As input for the experiments, initial neural policies $\Policy[e]$, certificates $\Certificate[e]$, and thresholds $\Bound[e]$ were obtained for each of the reach-avoid tasks corresponding to an edge $e$ in the graph $G$ as follows:  
\begin{enumerate}  
    \item Train a Proximal Policy Optimization (PPO) algorithm with the reward function defined as $5 \cdot \text{goal\_reached} - 5 \cdot \text{fail}$, where $\text{fail} = 1$ if the current state is in $\avoid$ and $0$ otherwise, and $\text{goal\_reached} = 1$ if the current state is in $\reach$ and $0$ otherwise.  
    \item For $\ChangedBound = 0.5$, attempt to find a certificate using the method described in \cite{badings_learning-based_2024}.  
    \begin{enumerate}  
        \item If certification fails within the maximum number of iterations, restart training from Step 1 using a new random seed. Repeat this step at most 10 times.  
    \end{enumerate}  
    \item Upon successful certification, initiate a binary search procedure with a lower bound of 0.5 and an upper bound of 1. Attempt certification using \cite{badings_learning-based_2024} at each step of the binary search to refine the threshold.  
\end{enumerate}  
Hyperparameters used for policy and certificate synthesis are given in \cref{tab:hyper_parameters_edge_policies}.

\begin{table}[t]
    \centering
    \begin{tabular}{ll}
        \toprule
        Parameter & Setting\\
        \midrule
        number of random samples            & 90 000\\
        number of counterexamples           & 30 000 \\
        epochs                              & 25\\
        batch size                          & 4096\\
        counterexample fraction             & 0.25\\
        counterexample refresh fraction     & 0.50\\
        optimizer                           & Adam \cite{kingma_adam_2014}\\
        learning rate $c$                   & $5 \cdot 10^{-4}$\\
        learning rate $\pi$                 & $5 \cdot 10^{-5}$\\
        $N$ loss learner                    & 16\\
        pretraining steps (PPO)             & $10^6$\\
        points init. verification grid      & $10^6$\\
        max refine factor                   & 10\\
        noise partition cells               & 144\\
        \textbf{loss mesh $\tau$}           & \textbf{0.001}\\
        \textbf{min. $L_\pi$ pretrain}      & \textbf{10}\\
        \bottomrule
    \end{tabular}
    \caption{Hyper-parameters used for policy and certificate synthesis with implementation from \protect\cite{badings_learning-based_2024}. Non-default values are marked in bold.}
    \label{tab:hyper_parameters_edge_policies}
\end{table}

\section{Acknowledgments}
This research was funded by the NWO Veni grant Explainable Monitoring (222.119). This work was done in part while Anna Lukina and Sterre Lutz were visiting the Simons Institute for the Theory of Computing.

\bibliographystyle{named}
\bibliography{Zotero,editable}

\end{document}